%% file: main.tex
\documentclass[11pt]{article} 
\usepackage{rldmsubmit,palatino}
\usepackage{graphicx}

\usepackage{amsmath}
\usepackage{amssymb}
\usepackage{amsthm}
\usepackage{thm-restate}
\usepackage{hyperref}

\def\eUDRL{eUDRL}
\newcommand{\subnote}[1]{{\color{gray}{#1\:}}}
\newcommand{\prob}{\mathbb{P}}
\newcommand{\stag}[1]{#1^{\Sigma}}
\newcommand{\argmax}{\mathrm{argmax}}

\newcommand{\indicator}{\mathbf{1}}
\DeclareMathOperator*{\ev}{\mathbb{E}}

\newcommand{\Traj}{\mathrm{Traj}}

\newcommand{\Seg}{\mathrm{Seg}}

\newtheorem{definition}{Definition}

\numberwithin{equation}{section}

\title{Upside-Down Reinforcement Learning Can Diverge \\ in Stochastic Environments With Episodic Resets}

\author{
Miroslav {\v{S}}trupl \thanks{Correspondence to \href{mailto:struplm@idsia.ch}{\texttt{struplm@idsia.ch}}}\enspace$^{1,2,3}$ \\
\And
Francesco Faccio $^{1,2,3}$ \\
\And
Dylan R.~Ashley $^{1,2,3}$ \\
\AND
J{\"{u}}rgen Schmidhuber $^{1,2,3,4,5}$ \\
\And
Rupesh Kumar Srivastava $^{5}$ \\
\AND
{\normalfont $^1$ The Swiss AI Lab IDSIA, Lugano, Switzerland} \\
$^2$ Universit{\`{a}} della Svizzera italiana (USI), Lugano, Switzerland \\
$^3$ Scuola universitaria professionale della Svizzera italiana (SUPSI), Lugano, Switzerland \\
$^4$ AI Initiative, King Abdullah University of Science and Technology (KAUST), Thuwal, Saudi Arabia \\
$^5$ NNAISENSE, Lugano, Switzerland \\
}

\begin{document}

\maketitle

\begin{abstract}
Upside-Down Reinforcement Learning (UDRL) is an approach for solving RL problems that does not require value functions and uses \emph{only} supervised learning, where the targets for given inputs in a dataset do not change over time \cite{schmidhuber2020reinforcement,srivastava2021training}.
Ghosh et al.~\cite{ghosh2020learning} proved that Goal-Conditional Supervised Learning (GCSL)---which can be viewed as a simplified version of UDRL---optimizes a lower bound on goal-reaching performance. 
This raises expectations that such algorithms may enjoy guaranteed convergence to the optimal policy in arbitrary environments, similar to certain well-known traditional RL algorithms.
Here we show that for a specific \emph{episodic} UDRL algorithm (eUDRL, including GCSL), this is not the case, and give the causes of this limitation.
To do so, we first introduce a helpful rewrite of eUDRL as a recursive policy update.
This formulation helps to disprove its convergence to the optimal policy for a wide class of stochastic environments.
Finally, we provide a concrete example of a very simple environment where eUDRL diverges.
Since the primary aim of this paper is to present a negative result, and the best counterexamples are the simplest ones, we restrict all discussions to finite (discrete) environments, ignoring issues of function approximation and limited sample size. 
\end{abstract}

\keywords{
upside-down reinforcement learning,
command-conditioned policies,
goal-conditioned supervised learning,
reinforcement learning
}

\acknowledgements{This work was supported by the European Research Council (ERC, Advanced Grant Number 742870), the Swiss National Supercomputing Centre (CSCS, Project s1090), and by the Swiss National Science Foundation (Grant Number  200021\_192356, Project NEUSYM). We also thank both the NVIDIA Corporation for donating a DGX-1 as part of the Pioneers of AI Research Award and IBM for donating a Minsky machine.}

\startmain 

\section{Introduction}

Upside-Down RL (UDRL) \cite{schmidhuber2020reinforcement} was introduced as an approach for solving RL problems that does not require value functions and uses only supervised learning.
In UDRL the agent takes (besides the state) an extra command input $(h,g)$. The components of the command input are referred to as the horizon $h$ and goal $g$.
The UDRL's command can be interpreted in many ways, e.g., ``achieve more then $g$ cumulative reward in at most $h$ steps".
However, in this paper we will discuss just following simple interpretation: ``reach goal $g$ in $h$ number of steps" (see definition \ref{de:CE} for more details).
Goal-conditioned supervised learning~\cite{ghosh2020learning} (GCSL) mainly differs from UDRL by restricting $h$ to some fixed horizon---thus leaving just the goal $g$ as input. Therefore GCSL could be understood as a special case of UDRL.
Originally, UDRL was introduced in many flavors: stochastic or deterministic environments, with episodic or single-life setting.
Here we aim to analyze the episodic setting for UDRL (\eUDRL{}). A discussion of \emph{single-life} UDRL \cite{schmidhuber2020reinforcement} is outside the scope of this paper.
Specifically, we will investigate the version of \eUDRL{} as described in \cite{srivastava2021training} up to some simplifications, e.g., we will not consider prioritized replay in order to simplify the analysis.
The results of our \eUDRL{} analysis also apply to GCSL.

\section{Background: Command Extension}
In this paper we will assume all random variable maps are defined on a measure space $(\Omega,\mathcal{F},\prob)$, with $\prob$ being a probability measure. Also, we will be dealing exclusively with discrete random variables.
This implies that all probability distributions can be expressed
as densities with respect to the suitable power of an arithmetic measure: these densities are essentially probabilities and so integration corresponds to summation.

A Markov Decision Process (MDP) is a tuple $\mathcal{M}=(\mathcal{S},\mathcal{A},p_T,\mu_0,r)$ \cite{puterman2014markov}, 
where we assume $\mathcal{S} \subset \mathbb{Z}^{n_S}$ and $\mathcal{A} \subset \mathbb{Z}^{n_A}$
to be finite state and action spaces. 
$p_T(s'|s,a) = \prob( S_{t+1} = s' \mid S_{t} = s,A_{t} = a )$ 
denotes the transition density (the subscript $T$ is an abbreviation of \emph{transition}, and does not refer to a random variable).
$\mu_0(s) = \prob( S_0=s )$ 
denotes the initial distribution.
$r(s',s,a)$ denotes the deterministic reward function, i.e., the reward at time $t$ is given by
$R_t = r(S_{t},S_{t-1},A_{t-1})$. 
We assume no discounting.
Further, we define the return $G_t = \sum_{k \in \mathbb{N}_0} R_{t+k+1}$ from time $t$.
An agent is determined by its (stochastic) policy $\pi(a|s) = \prob(A_t=a \mid S_t=s)$
, which is the
conditional probability of choosing an action $a$, given that the environment is in a state $s$.
The performance of an agent following a policy is measured by means of 
state $V^{\pi}(s) = \ev [G_t \mid S_t = s;\pi]$ and
action $Q^{\pi}(s,a) = \ev [G_t \mid S_t = s,A_t = a;\pi]$ value functions.
However, the symbols $\pi$, $V^{\pi}$, $Q^{\pi}$ will be used solely
here to denote policies and values of a special class of MDPs
called \textbf{Command Extensions}, which are introduced below.

The aim of training a \eUDRL{} agent is to make the agent better at fulfilling commands.
In order to discuss \eUDRL{}, where the agent takes an extra command input to produce an action, we simply extend the state space by the command to be able to view a
\eUDRL{} agent as an ordinary agent on a slightly bigger MDP. This motivates the following definition of Command Extension:
\begin{definition} \label{de:CE}
(Command Extension (CE))
of an MDP $\mathcal{M}=(\mathcal{S}, \mathcal{A}, p_T, r, \mu_0)$
is the MDP $\bar{\mathcal{M}} = (\bar{\mathcal{S}}, \mathcal{A}, \bar{p}_T, \bar{r}, \bar{\mu}_0,\rho)$, 
where:
\begin{itemize}
\item
$\rho:S \rightarrow \mathcal{G}$
is a goal map (usually just a projection on some state components).
We refer to $\mathcal{G} \subset \mathbb{Z}^{n_G}$ as the goal set.

\item
$\bar{\mathcal{S}} := \mathcal{S}\times\bar{\mathbb{N}}_0\times{\mathcal{G}}$, where for all 
$\bar{s} = (s,h,g) \in \bar{\mathcal{S}}$, $h$ has the meaning of remaining horizon and $g \in 
\mathcal{G}$ stands for a goal.
$\bar{\mathbb{N}}_0 := \{h \in \mathbb{N}_0 \mid h \leq N\}$ denotes a prefix of natural numbers, where $N$ denotes maximum remaining horizon.
The set of absorbing states $\bar{S}_A := \{ (s,h,g) \in \bar{\mathcal{S}} \mid h = 0 \}$ is formed
by all states with remaining horizon 0. The set of transition states is denoted $\bar{S}_T := \bar{\mathcal{S}} \setminus \bar{S}_A$.

\item
The initial distribution of $\bar{\mathcal{M}}$ is given by
$$\bar{\mu}_0(\bar{s}) := \bar{\mu}_0(s,h,g) := \prob(H_0=h ,G_0=g \mid S_0=s)\mu_0(s).$$

\item
The transition kernel is given by the density $\bar{p}_T$ which is defined as follows (we described just moves from transient states and only 
possibly non-zero cases, $\forall a \in \mathcal{A}$):
$$
\forall (s,h,g) \in \bar{\mathcal{S}}_T \: \forall s'\in \mathcal{S}:
\bar{p}_T ( (s',h-1,g) \mid (s,h,g), a )
= p_T(s' \mid s,a)
$$
We see that the remaining horizon decreases by $1$ till $0$ where an absorbing state
is entered. The goal once initialized is held fixed. Note that for the remaining horizon $h>0$, the transition dynamics in the original MDP state component is the same as in the original MDP.
\item
The nonzero reward is issued just when transiting to an absorbing state for the first time and is given by
$(\forall s,s'\in \mathcal{S}, \forall g\in\mathcal{G}, a\in \mathcal{A}):
\bar{r}((s',0,g),(s,1,g),a) := \indicator\{\rho(s') = g\}.$
\end{itemize}
\end{definition}

Note that in order to construct a command extension one has to supply (besides the original MDP) also 
$p_{H_0,G_0 \mid S_0}$
: the initial horizon and goal conditional.\footnote{CE also allows for
goal to be related to the return of the original MDP $\mathcal{M}$. This is
possible by extending the $\mathcal{M}$ state with a component that accumulates rewards.}

\paragraph{CE value functions \& optimal \eUDRL{} agents.}
The reward is defined such that it promotes increasing the probability of fulfilling commands. This can be observed by computing the value functions (see Appendix \ref{ap:CEvalues}
for details). \\ $\forall (s,h,g) \in \bar{S}_T, \forall a \in \mathcal {A}:$
\small
\begin{equation*}
    Q^{\pi}((s,h,g),a) =
\prob( \rho(S_{t+h})=g|S_{t}=s,H_{t}=h,G_{t}=g,A_t=a;\pi), \quad
V^{\pi}(s,h,g) = 
\prob(\rho(S_{t+h})=g|S_t=s,H_t=h,G_t=g;\pi ).
\end{equation*}
\normalsize

Thus, if we interpret "better at fulfilling commands" in the sense of an increased probability as \eUDRL{} claims, then the solution of the \eUDRL{} problem simply translates into finding an optimal agent for the Command Extension of the original MDP. This corresponds to the solution of the finite horizon reinforcement learning problem \cite{puterman2014markov}.\footnote{Notice that CE rewards and even values are bounded by 1.}
However, \eUDRL{} is not formulated in this way (as RL on corresponding CE) and is rather given explicitly as an iterative algorithm. 

\paragraph{Segment distribution}
The \eUDRL{} algorithm, which we will further describe, makes use of samples of trajectory segments. In order to be able to conduct a precise analysis, we have to first define the segment distribution that will be used. 
To be concise here, we leave the precise definition of the segment distribution $d_{\Sigma}^{\pi}$ and its properties to Appendix \ref{ap:SegmentDist}. Here we just mention some key facts which will be used in the later sections. By segments we mean continuous chunks of trajectories. Because of the remaining horizon and goal transition
dynamics of CE, a segment is given by its length $l(\sigma)$, the first remaining horizon $h_0^{\sigma}$ and the goal $g_0^{\sigma}$, the $l(\sigma)+1$ original MDP states $s_0^{\sigma},\ldots,s_{l(\sigma)}^{\sigma}$ and
the $l(\sigma)$ actions
$a_0^{\sigma},\ldots,a_{l(\sigma)-1}^{\sigma}$. In further writing
we will identify segments with such tuples
$\sigma=(l(\sigma),s_0^{\sigma},h_0^{\sigma},g_0^{\sigma},a_0^{\sigma},
s_1^{\sigma},a_1^{\sigma},\ldots,s_{l(\sigma)}^{\sigma})$.
The space of all such tuples will be denoted by $\Seg$. Formally,
we will denote by $\Sigma: \Omega \rightarrow \Seg$
(with components $\Sigma = (l(\Sigma),\stag{S}_0,\stag{H}_0,\stag{G}_0,\stag{A}_0,
\stag{S}_1,\stag{A}_1,\ldots,\stag{S}_{l(\Sigma)})$)
a random variable map with distribution $d_{\Sigma}^{\pi_n}$.
In appendix \ref{ap:SegmentDist} we show the following key properties of $d_{\Sigma}^{\pi_n}$ which will be used later
$(\forall\, 0 \leq i \leq k)$:
\small
\begin{equation*}
\prob(\stag{A}_i=a_i|\stag{S}_i=s_i,\ldots,\stag{S}_0=s_0,\stag{H}_0=h_0,\stag{G}_0=g,l(\Sigma)=k;\pi_n)
=\pi_n(a_i|s_i,h_0-i,g),
\tag{\ref{eq:segactions}}
\end{equation*}
\begin{align*}
\prob(\stag{S}_i=s_i|\stag{S}_{i-1}=s_{i-1},\stag{A}_{i-1}=a_{i-1},\ldots,\stag{S}_0=s_0,\stag{H}_0=h_0,\stag{G}_0=g_0,l(\Sigma)=k;\pi_n)
&=
p_T(s_i|s_{i-1},a_{i-1})
\tag{\ref{eq:segtransitions}}
\\
&=
\prob(\stag{S}_i=s_i|\stag{S}_{i-1}=s_{i-1},\stag{A}_{i-1}=a_{i-1}).
\end{align*}
\normalsize

\paragraph{\eUDRL{} as an iterative algorithm.}
\eUDRL{} is an iterative algorithm which 
starts with some initial policy $\pi_0$ and
generates a sequence of policies $(\pi_n)$.
Each iteration consists of two steps: in the first step a batch of trajectories is generated using the current policy $\pi_n$;
in the second step, \emph{segments} $\sigma \sim d_{\Sigma}^{\pi_n}$ of trajectories are sampled from the batch and
a new policy $\pi_{n+1}$ is fitted (by maximizing likelihood using supervised learning) to some action conditional of $d_{\Sigma}^{\pi_n}$. Precisely,
$$
\pi_{n+1} := \argmax_{\pi} \ev_{\sigma=
(l(\sigma),s_0^{\sigma},h_0^{\sigma},g_0^{\sigma},a_0^{\sigma},\ldots,s_{l(\sigma)}^{\sigma})
) \sim d_{\Sigma}^{\pi_n}} \log(\pi(a_0^{\sigma} \mid s_0^{\sigma},l(\sigma),\rho(s_{l(\sigma)}^{\sigma}))).
$$
The way that segments sampled from $d_{\Sigma}^{\pi_n}$ are used for learning is similar to that in Hindsight Experience Replay (HER) \cite{andrychowicz2018hindsight}.
Intuitively, given a segment sample $\sigma \sim d_{\Sigma}^{\pi_n}$, the first action $a_0^{\sigma}$ of the segment might be good for command $(l(\sigma),\rho(s_{l(\sigma)}^{\sigma}))$ which gets realized at the end of the segment independently of what command was
actually intended at the segment beginning $(h_0^{\sigma},g_0^{\sigma})$.
Thus a new policy $\pi_{n+1}$ is fitted using supervised learning to the conditional probability density $p_{\stag{A}_0 \mid \stag{S}_0, l(\Sigma), \rho(\stag{S}_{l(\sigma)}); \pi_n}$, i.e.,
$\pi_{n+1}(a \mid (s,h,g)) \doteq
\prob(\stag{A}_0=a \mid \stag{S}_0=s,l(\Sigma)=h,\rho(\stag{S}_{l(\sigma)})=g;\pi_n)
$, where $\doteq$ denotes equality up to fitting error.
Since we will not consider function approximation here, we can assume equality.

A good property of this iteration is that it is based purely on supervised learning (it does not feature value functions). The motivation behind this iteration was to make next policy ``better" than the previous one, but we will show later in this paper (Sections \ref{sec:nonoptimality} and \ref{sec:demo}) that it guarantees neither convergence to the optimum nor monotonic improvement of the policy in stochastic episodic environments (at least in sense of CE value functions).
It is interesting to note that \cite{schmidhuber2020reinforcement}
reserved the described iteration just for deterministic environments (where the algorithm converges) unlike
\eUDRL{}'s practical implementations \cite{srivastava2021training} and GCSL that were tested on stochastic environments.

\section{\eUDRL{} Recursion Rewrite}

As explained above, in every iteration, the \eUDRL{} algorithm fits a new policy $\pi_{n+1}$  to 
density $p_{\stag{A}_0|\stag{S}_0,l(\Sigma),\rho (\stag{S}_{l(\Sigma)}); \pi_n}$.
Since we are assuming no function approximation, we get ($\forall a\in \mathcal{A},(s,h,g)\in \bar{S}_T$):
\begin{equation}
\begin{aligned}
\pi_{n+1}(a|s,h,g)
&=
\prob(\stag{A}_0=a|\stag{S}_0=s,l(\Sigma)=h,\rho(\stag{S}_{l(\Sigma)}) = g; \pi_n)
\\
&=
\frac{
\prob(\rho (\stag{S}_{l(\Sigma)})=g | \stag{A}_0=a,\stag{S}_0=s,l(\Sigma)=h; \pi_n)
\prob(\stag{A}_0=a|\stag{S}_0=s,l(\Sigma)=h; \pi_n)
}
{
\prob(\rho (\stag{S}_{l(\Sigma)})=g|\stag{S}_0=s,l(\Sigma)=h; \pi_n)
}
\\
\Big(
&=
\frac{
Q_{A}^{\pi_n,g}(s,h,a)
\pi_{A,n} (a|s,h)
}
{
\sum_{a\in\mathcal{A}} Q_{A}^{\pi_n,g}(s,h,a) \pi_{A,n} (a|s,h)
}
\Big).
\end{aligned}
\label{eq:recursion}
\end{equation}
One could immediately see a correspondence with Reward-Weighted Regression "multiply by Q" recursion \cite{strupl2021reward} --- see the last rewrite in the parentheses. Although the
quantities in the numerator $\pi_{A,n}$ and $Q_{A}$ are not the actual policies and Q-values, we will denote them by similar
symbols to resemble this correspondence. Let us define the
``average" $Q$-value
$
Q_{A}^{\pi_n,g}(s,h,a) := 
\prob(\rho (\stag{S}_{l(\Sigma)})=g | \stag{A}_0=a, \stag{S}_0=s,l(\Sigma)=h; \pi_n)
$
and the
``average" policy
\begin{equation}
\begin{aligned}
\pi_{A,n} (a|s,h) &:= \prob(\stag{A}_0=a|\stag{S}_0=s,l(\Sigma)=h; \pi_n)
=
\sum_{h' \in \mathbb{N}_0,g'\in\mathcal{G}}
\prob(\stag{A}_0=a, \stag{H}_0=h', \stag{G}_0=g' |\stag{S}_0=s,l(\Sigma)=h; \pi_n)
\\
&=
\sum_{h'\geq h,g'\in\mathcal{G}}
\prob(\stag{A}_0=a| \stag{H}_0=h', \stag{G}_0=g', \stag{S}_0=s,l(\Sigma)=h; \pi_n)
\prob( \stag{H}_0=h', \stag{G}_0=g' |\stag{S}_0=s,l(\Sigma)=h; \pi_n)
\\
&=
\sum_{h'\geq h,g'\in\mathcal{G}}
\pi_n (a|h',g',s)
\prob( \stag{H}_0=h', \stag{G}_0=g' |\stag{S}_0=s,l(\Sigma)=h; \pi_n)
.\label{eq:apolicy}
\end{aligned}
\end{equation}
First, for $h'<h$ the summed expression is actually zero (a segment must fit into a trajectory,
thus the remaining horizon at the begin of a segment $\stag{H}_0$ must be always greater or equal
then the length of the segment).
Finally, the fact that $\prob(\stag{A}_0=a| \stag{H}_0=h', \stag{G}_0=g', \stag{S}_0=s,l(\Sigma)=h; \pi_n)
 = \pi_n(a|h',g',s)$ is just equation \ref{eq:segactions}.
The ``average" in names refers to the fact that some of the CE state components are missing, thus in the case of average action-value function $Q_{A}^{\pi_n,g}$ we have one common value for all goals $\stag{G}_0$ etc.
For the average policy $\pi_{A,n}$ it is even more evident as the average policy is obtained through marginalization of the goal and horizon from current policy $\pi_n$.

\section{\eUDRL{} Non-Optimality in Stochastic Environments}\label{sec:nonoptimality}

The fact that \eUDRL{} hardly converges in stochastic environments follows directly from the above introduced
\eUDRL{} recursion rewrite.
We will state it formally in the following lemma, which we must motivate first.
Assume a scenario with one optimal policy and where the recursion has a limit (we would
like it to be the optimal policy). 
As a consequence $(\pi_{A,n})$ approaches another limit and also does $Q_A$. Thus we can imagine that
up to some arbitrary small errors $\pi_n, \pi_{A,n}, Q_{A}^{\pi_n,g}$ are fixed at their limits
(for sufficiently large $n$).
Suppose there are two goals for which their corresponding optimal policies have disjoint supports
at some state (simply for this two goals there does not exists a policy which is
optimal for both of them). Then since the average policy shares the same value for all goals, the difference has to be accounted in the multiplication by $Q_A$. Thus the average Q has to be
very selective. This is actually true just in a deterministic environment.
As the stochasticity of the environment increases, the action value functions tend to be more
and more flat. This inevitably causes \eUDRL{}'s inability to reach the optimal policy.
Before stating the lemma, notice that (using \eqref{eq:segtransitions}):
$
Q_{A}^{\pi_n,g}(s,1,a) := 
\prob(\rho (\stag{S}_{l(\Sigma)})=g | \stag{A}_0=a, \stag{S}_0=s,l(\Sigma)=1; \pi_n)
=
\sum_{s'\in \rho^{-1}(\{g\})}
\prob(\stag{S}_{l(\Sigma)}=s' | \stag{A}_0=a, \stag{S}_0=s,l(\Sigma)=1; \pi_n)
=
\sum_{s'\in \rho^{-1}(\{g\})}
p_T (s'|a,s)
$
is policy independent (because we are at horizon 1).

\begin{restatable}[]{lemma}{nonoptproof}\label{le:nonoptproof}
(\eUDRL{} insensitivity to desired goal input at remaining horizon 1)
Let us have an original MDP $\mathcal{M}=(\mathcal{S}, \mathcal{A}, p_T, r, \mu_0)$ and its command extension $\bar{\mathcal{M}}=(\bar{\mathcal{S}}, \mathcal{A}, \bar{p}_T, \bar{r}, \bar{\mu}_0,\rho)$, 
such that there exists a state $s \in \mathcal{S}$ and two goals $g_0\neq g_1$, $g_0,g_1 \in \mathcal{G}$ such that 
$M_0 := \argmax_{a\in \mathcal{A}} Q^*((s,1,g_0),a)$ and $M_1 := \argmax_{a\in \mathcal{A}} Q^*((s,1,g_1),a)$ (optimal policy supports for $g_0,g_1$)
have empty intersection $M_0 \cap M_1 = \emptyset$.
Assume
$
Q_{A}^{\pi_n,g_i}(s,1,a) \geq q_i(1 - \delta)
$ where delta $\delta >0$ and $q_i := \max_a Q_{A}^{\pi_n,g_i}(s,1,a)$.
(It is useful to assume $\delta$ to be the tightest possible of such a bound.)
Then, when $\delta < 1$ (stochastic environment), the sequence $(\pi_n)$ of policies
produced by \eUDRL{} recursion cannot tend to the optimal policy set.
\end{restatable}

A proof of the lemma is included in Appendix \ref{ap:nonoptproof}. Our aim 
(in this example of bad \eUDRL{} behaviour) was not to be exhaustive but rather to illustrate
what can go wrong.
Evidently the lemma does not cover the case where the optimal policy supports happen to be disjoint just for horizons $> 1$.
Further since the \eUDRL{} computes the average policy also by averaging through different horizons (not just goals), there could be errors attributed
to insensitivity to the remaining horizon part of the command. These are not used in the lemma.
Obviously, there can be scenarios where the situation described in the lemma does not
hold, but one can (motivated by the recursion rewrite \eqref{eq:recursion}) construct more robust counterexamples (allowing for different
horizon then 1 etc.) with more complicated proofs. 

\section{Demonstration}\label{sec:demo}
The situation from the lemma is demonstrated in the following, trivial, horizon 1
example. Let us have an original MDP $\mathcal{M} = (\mathcal{S},\mathcal{A},p_T,\mu_0,r)$, where the state and action spaces are $\mathcal{S} = \{0,1\}$ and
$\mathcal{A} = \{0,1\}$ with initial state fixed at
$s_0 = 0$.
The command extension (CE) of $\mathcal{M}$, $\bar{\mathcal{M}}$, bounds the remaining horizons by $N=1$. The initial remaining horizon
is fixed at $h_0 = 1$.
The goal set is $\mathcal{G} := \mathcal{S}$, $\rho:=\mathrm{id}_{\mathcal{S}}$ ($\mathrm{id}_{\mathcal{S}}$ is the identity map on $\mathcal{S}$), and the initial goal is 
 uniformly distributed.
Since the original MDP $\mathcal{M}$ has only one initial state $s_0=0\in \mathcal{S}$ and the initial remaining horizon is always $h_0 = 1$ it suffices to
describe the transition kernel $p_T$ just from state $0$.
When action $a_0 = 0$ is chosen, $\mathcal{M}$ remains in $0$ with probability $\alpha$
and transits to $1$ with probability $1-\alpha$.
When action $a_0 = 1$ is chosen, $\mathcal{M}$ remains in $0$ with probability $1-\alpha$
and transits to $1$ with probability $\alpha$.
Thus the stochasticity of the environment is controlled by the parameter $\alpha \in [0.5,1]$, where $\alpha = 1$ forces deterministic transitions, while $\alpha = 0.5$ means that actions do not affect transitions at all (the transition kernel becomes uniform).
The CE has only two initial states: $(0,1,0)$ ($\in \bar{\mathcal{S}} = \mathcal{S}\times \mathbb{N}_0 \times \mathcal{G}$) and $(0,1,1)$ differing only in the goal component.
These are also its only transient states, thus it suffices to evaluate values and
policies only for them.

In Figure~\ref{fig:demo} (on the left) we plot the RMSVE between $V^{\pi_n}$ and $V^*$ and (in the center) the supremum distance between $\pi_n$ and $\pi^*$.
We can see that after the first iteration everything remain constant 
up to slight deviations due to the limited sample size (for this trivial example everything can be verified by hand computations as done in Appendix \ref{ap:demonotes}). Moreover the RMSVE and the supremum distance do not approach 0
anymore, e.g., for $\alpha  = 0.6$ (heavy stochasticity) the supremum distance to the
optimal policy remains above $0.3$.
This error, as proved, is inherent to \eUDRL{}/GCSL methods and cannot be alleviated
by increasing the sample size or the number of iterations, unlike classical RL methods (applied to CE $\bar{\mathcal{M}}$), which are guaranteed to converge to the global optimum.
In addition, Figure~\ref{fig:demo} (right) shows the GCSL goal reaching objective $J(\pi_n) = \sum_{\bar{s}\in \bar{\mathcal{S}}} V^{\pi_n}(\bar{s})\bar{\mu}_0(\bar{s})$. In this example
$\alpha=0.6$ and the two different initial conditions are chosen. We can observe
a huge decrease in $J(\pi_n)$ when starting from the optimal policy, thus there
is no monotonic improvement in $J(\pi_n)$.
The source code for an implementation of this example is available at \url{https://github.com/struplm/UDRL-GCSL-counterexample.git}
\begin{figure}
    \centering
    \includegraphics[width=6cm]{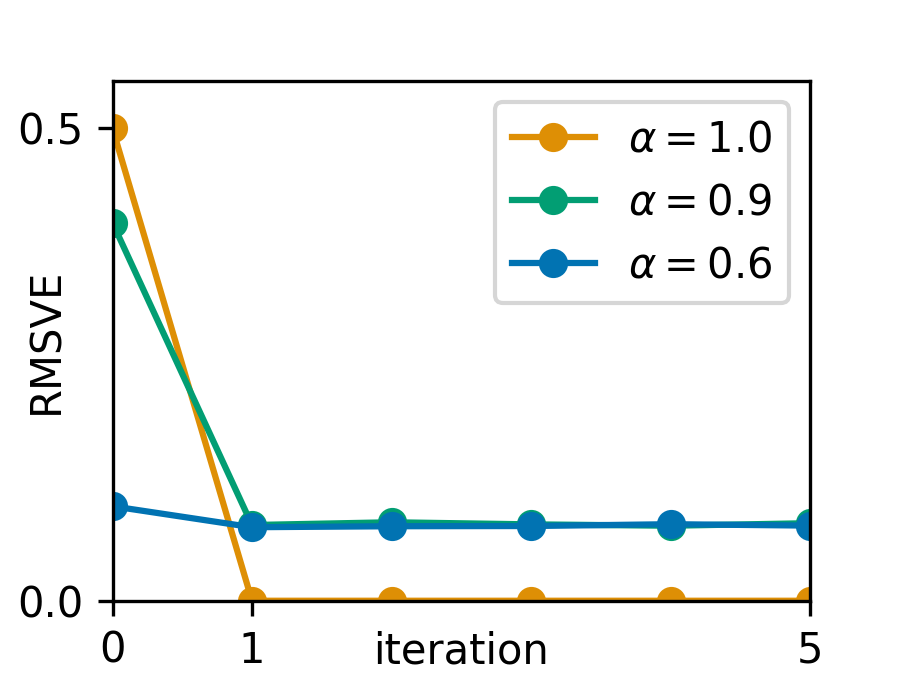}
    \includegraphics[width=6cm]{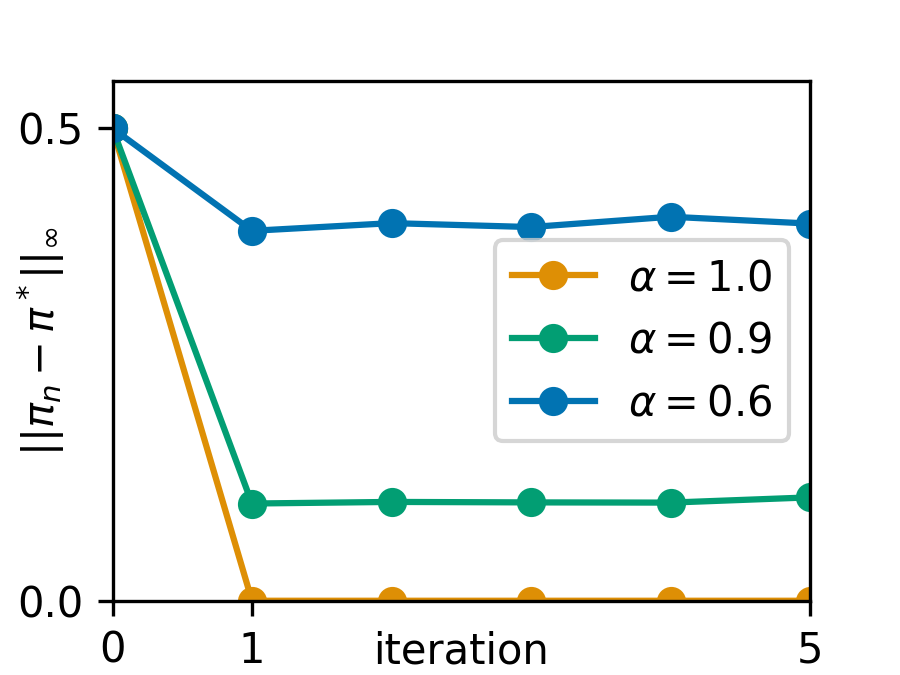}
    \includegraphics[width=6cm]{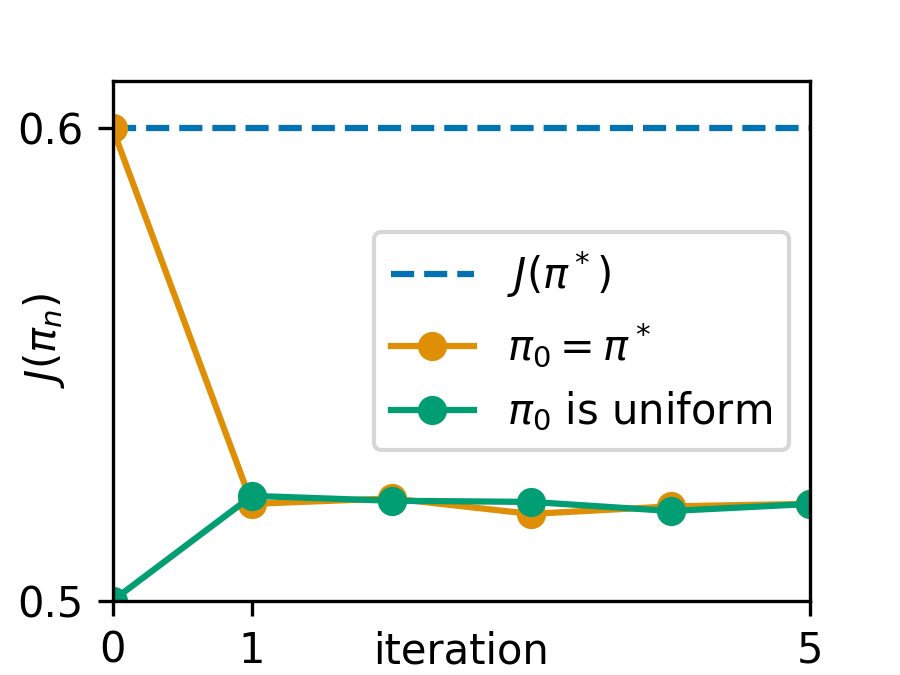}
\caption{
(left and center) Root mean square value error (RMSVE) between current and optimal value and
supremum distance of current policy $\pi_n$ to the optimal policy $\pi^*$ for various values of $\alpha$,
(right) GCSL goal reaching objective $J(\pi_n)$ with $\alpha=0.6$ and two initial conditions
$\pi_0 = \pi^*$ (the optimal policy) and $\pi_0$ uniform. A batch of 10000 trajectories (with length 1) was generated per each \eUDRL{} iteration.
}
    \label{fig:demo}
\end{figure}

\section{Conclusion}
We presented the concept of Command Extension of an MDP (see Definition~\ref{de:CE})
which, together with the formal definition of Segment Distribution (see Appendix~\ref{ap:SegmentDist}), allows for theoretical investigation of iterative, episodic variants of Upside-Down Reinforcement Learning 
algorithms (\eUDRL{}/GCSL).
A key result is the \eUDRL{} recursion rewrite \eqref{eq:recursion}
which helps to understand causes of non-optimal behavior of the aforementioned 
algorithms.
We disproved eUDRL's convergence to the optimum for quite a large class of stochastic
environments in Lemma~\ref{le:nonoptproof}.
A demonstration of non-optimal behavior was also provided through a concrete example.
Besides showing convergence to a non-optimal policy, the example also demonstrates
that there is no guarantee for monotonic improvement. 

Our analysis  focuses on the episodic setting (\eUDRL{}). A discussion of the highly promising \emph{single-life} UDRL \cite{schmidhuber2020reinforcement} is outside the scope of this work.
The fact that \eUDRL{} converges in \emph{deterministic}  environments is trivial to show, and this convergence was one of
the main motivations behind introducing the UDRL concept \cite{schmidhuber2020reinforcement}. The convergence of \eUDRL{} in the deterministic GCSL sub-case was also commented on   \cite{ghosh2020learning}.
Re-labeling history \ref{eq:recursion} causes \eUDRL{}/GCSL to fail to find an optimal policy in stochastic environments.
This result applies to certain existent implementations \cite{srivastava2021training,ghosh2020learning} that nevertheless produce useful results in practice.
It remains an open question to characterize the conditions under which these implementations learn policies that are useful despite being non-optimal.

\bibliographystyle{plain} 
\bibliography{main}

\newpage
\appendix
\input{appendix.tex}

\end{document}

%% file: appendix.tex
\section{Command Extension Values}
\label{ap:CEvalues}

Let $\mathcal{M} = (\mathcal{S},\mathcal{A},p_T,\mu_0,r)$ be an MDP and
$\bar{\mathcal{M}} = (\bar{\mathcal{S}}, \mathcal{A}, \bar{p}_T, \bar{r}, \bar{\mu}_0,\rho)$ be its CE. We aim to prove:
for all $(s,h,g) \in \bar{S}_T$ and all actions $a \in \mathcal {A}$ it holds:
$$
\begin{aligned}
Q^{\pi}((s,h,g),a) &=
\prob(\rho(S_{t+h})=g|S_{t}=s,H_{t}=h,G_{t}=g,A_t=a,\pi ), \\
V^{\pi}(s,h,g) &=
\prob (\rho(S_{t+h})=g|S_t=s,H_t=h,G_t=g,\pi ).
\end{aligned}
$$
The statement follows directly from the CE definition, we just have to realize that, due to the
deterministic dynamics in remaining horizon and the reward definition, there is a non-zero reward only after $h$ time-steps, where $h$ is the remaining horizon component of the CE state, i.e.,
$$
\begin{aligned}
Q^{\pi}((s,h,g),a)
&=
\ev [G_t \mid (S_t,H_t,G_t)=(s,h,g),A_t = a;\pi]
\\&=
\ev [ \sum_{k \in \mathbb{N}_0} R_{t+k+1} \mid (S_t,H_t,G_t)=(s,h,g),A_t = a;\pi]
\\&=
\ev [ \sum_{k \in \mathbb{N}_0} r((S_{t+k+1},H_{t+k+1},G_{t+k+1}),(S_{t+k},H_{t+k},G_{t+k}),A_{t+k}) \mid (S_t,H_t,G_t)=(s,h,g),A_t = a;\pi]
\\&=
\ev [ r((S_{t+h},H_{t+h},G_{t+h}),(S_{t+h-1},H_{t+h-1},G_{t+h-1}),A_{h-1}) \mid (S_t,H_t,G_t)=(s,h,g),A_t = a;\pi]
\\&=
\ev [ r((S_{t+h},0,g),(S_{t+h-1},1,g),A_{t+h-1}) \mid (S_t,H_t,G_t)=(s,h,g),A_t = a;\pi]
\\&=
\ev [ \indicator \{\rho(S_{t+h})=g\} \mid (S_t,H_t,G_t)=(s,h,g),A_t = a;\pi]
\\&=
\prob ( \rho(S_{t+h})=g \mid (S_t,H_t,G_t)=(s,h,g),A_t = a;\pi).
\end{aligned}
$$
Similarly for $V^{\pi}(s)$.

\section{Segment Distribution and its Factorization}
\label{ap:SegmentDist}
There is no need to include rewards in trajectories (or their segments)
because they are deterministic.
Since actions assigned by a policy $\pi$ on $S_A$ do not matter (they are no longer affecting transitions) we will
without loss of generality restrict to policies $\pi$, which generates just single action on $S_A$., i.e.
$\forall \bar{s} \in S_A:\:\pi(a_A|\bar{s}) = 1$, where $a_A \in \mathcal{A}$ is some
a priori fixed action which is common to all policies. In this section we assume that all trajectories are generated by using one fixed policy $\pi$.
The length $l(\cdot)$ of a trajectory is defined as the number of transitions
till an absorbing state is entered for the first time.
Regarding trajectories we can consider just prefixes of length $N$ because the
CE MDP bounds the remaining horizon component by $N$. Further we can
restrict to the subspace $\Traj \subset (\bar{\mathcal{S}}\times\mathcal{A})^N\times\bar{\mathcal{S}}$ allowed
by remaining horizon/goal dynamics of CE, i.e. for a trajectory
$\tau = (\bar{s}_0,a_0,\ldots,\bar{s}_N) \in \Traj$, $\bar{s}_t = (s_t,h_t,g_t)$, $0\leq t\leq N$
the remaining horizon $h_t$ decreases by 1
from its initial value $h_0=l(\tau)$ till 0 when entering to
an absorbing state. The goal $g_t=g_0$ remains in its initial value.
Thus $\tau$ is fully determined by the initial horizon $h_0=l(\tau)$, the initial goal $g_0$, the
states $s_0,\ldots,s_{l(\tau)}$ and actions $a_0,\ldots,a_{l(\tau)-1}$, i.e.
$\tau = ((s_0,h_0,g_0),a_0,(s_1,h_0-1,g_0),a_1,\ldots,(s_{l(\tau)},0,g_0),a_A,\ldots,(s_{l(\tau)},0,g_0))$.
The probability of $\tau$ is given as:
$$
\prob( \mathcal{T} = \tau;\pi)
=
\left( \prod_{t=1}^{l(\tau)}
p_T(s_t | a_{t-1}, s_{t-1})
\right)
\times
\left( \prod_{t=0}^{l(\tau)-1}
\pi(a_t|\bar{s}_t)
\right)
\times
\prob( G_0=g_0,H_0=h_0|S_0=s_0)
\mu_0(s_0),
$$
where $\mathcal{T}:\Omega\rightarrow\Traj$,$\mathcal{T} = ((S_0,H_0,G_0),A_0,\ldots,(S_N,H_N,G_N))$ is the trajectory random variable map.

\paragraph{Segment distribution}
Segments are assumed to be continuous chunks of trajectories in $\Traj$,
so they respect the CE horizon/goal dynamics.
We assume them to be always contained within the length of a trajectory.
Notice that \eUDRL{} is learning actions just in a state which happens to be the first state of a segment.
Since learning actions in absorbing states is not meaningful, we will assume the first state of a segment to be transient ($\in \bar{S}_T$)\footnote{It would be even better to additionally assume the original MDP component of this state to be transient for the similar reason.}.
Thus, segments $\sigma$ are fully determined by the
segment length, denoted by $l(\sigma)$ (number of transitions), by the
remaining horizon and goal at the segment beginning
$h_0^{\sigma}$,$g_0^{\sigma}$, by and $l(\sigma)+1$ states and $l(\sigma)$
actions. Without loss of generality we will identify $\sigma$ with such a tuple
$\sigma=(l(\sigma),s_0^{\sigma},h_0^{\sigma},g_0^{\sigma},a_0^{\sigma},s_1^{\sigma},a_1^{\sigma},\ldots,s_{l(\sigma)}^{\sigma})$.
The space of all such tuples will be denoted $\Seg$.
Formally, we will assume that there exists a random variable map $\Sigma: \Omega \rightarrow \Seg$, $\Sigma = (l(\Sigma),\stag{S}_0,\stag{H}_0,\stag{G}_0,\stag{A}_0,\stag{S}_1,\stag{A}_1,\ldots,\stag{S}_{l(\Sigma)})$
with distribution $d_{\Sigma}^{\pi}$ given
below.

We construct the segment distribution $d_{\Sigma}^{\pi}$ in a similar way as the state visitation distribution
--- summing across appropriate trajectory distribution marginals.
This causes the result to be un-normalized (among other restrictions, e.g., on the first state of the segment), therefore we have to include a normalization constant $c$. ($\forall \sigma \in \Seg$):
$$
\begin{aligned}
(d_{\Sigma}^{\pi_n}(\sigma) :=)\:
\prob(\Sigma=\sigma;\pi)
&= c\sum_{t \leq N -l(\sigma)}
\prob(S_t=s_0^{\sigma},H_t=h_0^{\sigma},G_t=g_0^{\sigma},A_t=a_0^{\sigma}, \ldots, S_{t+l(\sigma)}=s_{l(\sigma)}^{\sigma} ;\pi)
\\
&=
c
\sum_{t \leq N -l(\sigma)}
\left( \prod_{i=1}^{l(\sigma)}
p_T(s_{i}^{\sigma} | a_{i-1}^{\sigma}, s_{i-1}^{\sigma})
\right)
\times
\left( \prod_{i=0}^{l(\sigma)-1}
\pi(a_i^{\sigma}|\bar{s}_i^{\sigma})
\right)
\times
\prob(S_t=s_0^{\sigma},H_t=h_0^{\sigma},G_t = g_0^{\sigma};\pi)
\\
&=
c
\left( \prod_{i=1}^{l(\sigma)}
p_T(s_{i}^{\sigma} | a_{i-1}^{\sigma}, s_{i-1}^{\sigma})
\right)
\times
\left(\prod_{i=0}^{l(\sigma)-1}
\pi(a_i^{\sigma}|\bar{s}_i^{\sigma})
\right)
\times
\sum_{t \leq N -l(\sigma)}
\prob(S_t=s_0^{\sigma},H_t=h_0^{\sigma},G_t = g_0^{\sigma};\pi)
\end{aligned}
$$
Given the factorized form of the segment distribution we can conclude the following (through computing marginals and calculating conditional density ratios):
\small
\begin{equation*}
\prob(l(\Sigma)=k,\stag{S}_0=s_0,\stag{H}_0=h_0,\stag{G}_0=g;\pi)
=
c
\sum_{t \leq N -k}
\prob(S_t=s_0,H_t=h_0,G_t = g;\pi)
\end{equation*}

\begin{equation*}
\prob(\stag{A}_0=a_0,\stag{S}_1=s_1,\stag{A}_1=a_1,\ldots,\stag{S}_k=s_k
|\stag{S}_0=s_0,\stag{H}_0=h_0,\stag{G}_0=g_0,l(\Sigma)=k
;\pi)
=
\left( \prod_{i=1}^{k}
p_T(s_{i} | a_{i-1}, s_{i-1})
\right)
\times
\prod_{i=0}^{k-1}
\pi(a_i|\bar{s}_i)
\end{equation*}
and
$(\forall 0 \leq i \leq k)$:
\begin{align}
\prob(\stag{A}_i=a_i|\stag{S}_i=s_i,\ldots,\stag{S}_0=s_0,\stag{H}_0=h_0,\stag{G}_0=g,l(\Sigma)=k;\pi)
&=\pi(a_i|s_i,h_0-i,g)
\label{eq:segactions} \\
&=
\prob(\stag{A}_i=a_i|\stag{S}_i=s_i,\stag{H}_0=h_0,\stag{G}_0=g;\pi) \nonumber
\end{align}
\begin{align}
\prob(\stag{S}_i=s_i|\stag{S}_{i-1}=s_{i-1},\stag{A}_{i-1}=a_{i-1},\ldots,\stag{S}_0=s_0,\stag{H}_0=h_0,\stag{G}_0=g_0,l(\Sigma)=k;\pi)
&=
p_T(s_i|s_{i-1},a_{i-1})
\label{eq:segtransitions}
\\
&=
\prob(\stag{S}_i=s_i|\stag{S}_{i-1}=s_{i-1},\stag{A}_{i-1}=a_{i-1})
\nonumber
\end{align}
\normalsize
After defining $\stag{H}_i := \stag{H}_0-i,\stag{G}_i := \stag{G}_0$, we can write
$\prob(\stag{A}_i=a_i|\stag{S}_i=s_i,\stag{H}_i=h_i,\stag{G}_i=g_i;\pi)= \pi(a_i|s_i,h_i,g_i)
$.

\section{\eUDRL{} non-optimality in Stochastic Environments}
\label{ap:nonOpt}

\nonoptproof*
\begin{proof}
\label{ap:nonoptproof}
Fix an arbitrary $ 0 < \epsilon < 1$
so that $\frac{1-\epsilon}{\epsilon}(1-\delta) > 1$ (since $(1-\delta) > 0$ we can make $\frac{1-\epsilon}{\epsilon}$ arbitrarily large
by choosing $\epsilon >0$ arbitrarily small).
Assume that the sequence $(\pi_n)$ tends to the optimal policy
set.
Then there exists $n_0$ such that for all $n > n_0$ it holds:
\begin{align*}
\sum_{a\in M_0}\pi_{n+1}(a|s,1,g_0) &\geq 1-\epsilon&
\sum_{a\in \mathcal{A}\setminus M_0}\pi_{n+1}(a|s,1,g_0) &\leq \epsilon
\\
\sum_{a\in M_1}\pi_{n+1}(a|s,1,g_1) &\geq 1-\epsilon&
\sum_{a\in \mathcal{A}\setminus M_1}\pi_{n+1}(a|s,1,g_1) &\leq \epsilon
\end{align*}
From \eUDRL{} recursion rewrite \eqref{eq:recursion} we can bound $\pi_{n+1}$ terms from above and below using
$\pi_{A,n}$, $q_i$ and $\delta$ (for all $n>n_0$):
\begin{align*}
\sum_{a\in M_0} c_0 q_0 \pi_{A,n}(a|s,1) &\geq 1-\epsilon&
\sum_{a\in \mathcal{A}\setminus M_0} c_0 q_0(1-\delta) \pi_{A,n}(a|s,1) &\leq \epsilon
\\
\sum_{a\in M_1}c_1 q_0 \pi_{A,n}(a|s,1) &\geq 1-\epsilon&
\sum_{a\in \mathcal{A}\setminus M_1}c_1 q_0(1-\delta) \pi_{A,n}(a|s,1) &\leq \epsilon
\end{align*}
where we introduced $c_i^{-1} := \sum_{a\in\mathcal{A}} Q_{A}^{\pi_n,g_i}(s,1,a) \pi_{A,n} (a|s,1) $ the normalizing denominators of the recursion.
Further denote
$\beta_0 := \sum_{a\in M_0}\pi_{A,n}(a|s,1)$ and
$\beta_1 := \sum_{a\in M_1}\pi_{A,n}(a|s,1)$.
Notice that since
$M_0$ and $M_1$ are disjoint then $\sum_{a\in \mathcal{A}\setminus (M_0 \cup M_1)} \pi_{A,n}(a|s,1) = 1-\beta_0-\beta_1$.
We will assume $0<\beta_i<1$, $i=1,2$ with the licence that the edge cases are somewhat trivial,
e.g., $\beta_0 = 0$ makes the first equation from the last display-math impossible or
$\beta_0 = 1$ make $\beta_1 = 0$ (again through disjointness) and we can argument similarly etc..
Thus we get:
\begin{align*}
c_0 q_0 \beta_0 &\geq 1-\epsilon&
c_0 q_0(1-\delta) (1-\beta_0) &\leq \epsilon
\\
c_1 q_0 \beta_1 &\geq 1-\epsilon&
c_1 q_0(1-\delta) (1-\beta_1) &\leq \epsilon
\end{align*}
Dividing the equations in the first row and then in the second row we get:
\begin{align*}
\frac{\beta_0}{(1-\beta_0)} &\geq \frac{1-\epsilon}{\epsilon} (1-\delta)
\\
\frac{\beta_1}{(1-\beta_1)} &\geq \frac{1-\epsilon}{\epsilon}(1-\delta).
\end{align*}
Now from the definition of $\beta_0$ and $\beta_1$ and the disjointness of $M_0$ and $M_1$,
we have that $1-\beta_0 \geq \beta_1$ and  $\beta_0 \leq 1-\beta_1$, which gives:
$$
\frac{1-\beta_0}{\beta_0} \geq \frac{\beta_1}{1-\beta_1} \geq \frac{1-\epsilon}{\epsilon}(1-\delta).
$$
Thus, to sum up, it must hold at the same time:
$$
\frac{\beta_0}{1-\beta_0} \geq \frac{1-\epsilon}{\epsilon}(1-\delta)
\quad\quad
\left(\frac{\beta_0}{1-\beta_0}\right)^{-1} = \frac{1-\beta_0}{\beta_0} \geq \frac{1-\epsilon}{\epsilon}(1-\delta)
$$
But it is impossible, since we make $\frac{1-\epsilon}{\epsilon}(1-\delta) > 1$ a priori
and it is impossible
for a number $\frac{\beta_0}{1-\beta_0}$ and its inversion to both be greater then 1.
\end{proof}

\section{Notes on Demonstration}
\label{ap:demonotes}
In this note we aim to compute exact values of $\pi_n$ for the example used in the demonstration.
As already mentioned it suffices to evaluate policies just for $s=0$, $h=1$ (the only transient states of the CE). We will note also the variable names inside densities (in case they were substituted by some concrete values), so it is not necessary to always look for how we fixed individual variable positions.\footnote{
This is just to improve readability of this paragraph by helping a reader to keep track of substitutions. We use a different color (gray) for these extra notes.
}
The result is: $\pi_n(a|\subnote{s=}0,\subnote{h=}1,g) = p_T(g|a,\subnote{s=}0)$ for $n>0$.
This result follows directly from the recursion \eqref{eq:recursion}. From the note
before Lemma~\ref{le:nonoptproof} we already know that for horizon 1 (our case) it
holds $Q_{A}^{\pi_n,g} (\subnote{s=}0,\subnote{h=}1,a) = \sum_{s'\in\rho^{-1}(\{g\})} p_T(s'|a,\subnote{s=}0) = p_T(g|a,\subnote{s=}0)$
where we used $\mathcal{G} = \mathcal{S}$, $\rho = \mathrm{id}_{\mathcal{S}}$.
Since for fixed horizon 1 (our case) segments are just trajectories (we fixed the minimum length of segment to 1 because of the requirement  the first state to be transient), the distribution of the first CE state of a segment is the same initial distribution of the trajectory. Thus the following term
from \eqref{eq:apolicy} simplifies:
$\prob(\stag{H}_0=1,\stag{G}_0=g'| \stag{S}_0=0,l(\Sigma)=1 ; \pi_n )
= \prob(G_0=g') = 0.5$. Thus we get the following relation for the ``average" policy:
$\pi_{A,n}(a|\subnote{s=}0,\subnote{h=}1) = 0.5 \sum_{g' \in \mathcal{G}} \pi_n(a|\subnote{s=}0,\subnote{h'=}1,g') $.
Notice that the optimal policy $\pi^{*}$ is given by $\pi^{*}(a|\subnote{s=}0,\subnote{h=}1,g) = 1$ for $a=g$.
Notice that the optimal policy $\pi^{*}$, the
uniform policy (our only candidates for the initial condition $\pi_0$) and $p_T(g|a,\subnote{s=}0)$ (our candidate for $pi_n$, $n>0$) all have the following symmetry: $\exists \beta \in [0,1]$ so that $f(a,g) = \beta$ for $a=g$, and
$f(a,g) = 1-\beta$ otherwise.
Further notice that due to the symmetry the transition probability $p_T(g|a,\subnote{s=}0)$ is also a conditional probability distribution over actions given a goal $g$ (summing over actions gives 1).
Now it suffices to prove that if $\pi_n$ has the symmetry then $\pi_{n+1} = p_T$.
Thus assume $\pi_n$  has the symmetry
we get  $\pi_{A,n}(a|\subnote{s=}0,\subnote{h=}1) = 0.5 \sum_{g' \in \mathcal{G}} \pi_n(a|\subnote{s=}0,\subnote{h'=}1,g') =
0.5 (\beta + 1-\beta) = 0.5$.
Thus in \eqref{eq:recursion} we are just multiplying
$Q_{A}^{\pi_n,g} (\subnote{s=}0,\subnote{h=}1,a) = p_T(g|a,\subnote{s=}0)$ by a constant $\pi_{A,n}(a|\subnote{s=}0,\subnote{h=}1)=0.5$.
After normalization by denominator we must get back to $\pi_{n+1}(a|\subnote{s=}0,\subnote{h=}1,g) = p_T(g|a,\subnote{s=}0)$ (here we used the fact that $p_T$ is also conditional distribution over actions given a goal).

%% file: main.bbl
\begin{thebibliography}{1}

\bibitem{andrychowicz2018hindsight}
Marcin Andrychowicz, Filip Wolski, Alex Ray, Jonas Schneider, Rachel Fong,
  Peter Welinder, Bob McGrew, Josh Tobin, Pieter Abbeel, and Wojciech Zaremba.
\newblock Hindsight experience replay, 2018.

\bibitem{ghosh2020learning}
Dibya Ghosh, Abhishek Gupta, Ashwin Reddy, Justin Fu, Coline Devin, Benjamin
  Eysenbach, and Sergey Levine.
\newblock Learning to reach goals via iterated supervised learning, 2019.

\bibitem{puterman2014markov}
Martin~L Puterman.
\newblock {\em Markov Decision Processes: Discrete Stochastic Dynamic
  Programming}.
\newblock John Wiley \& Sons, 2014.

\bibitem{schmidhuber2020reinforcement}
Juergen Schmidhuber.
\newblock Reinforcement learning upside down: Don't predict rewards -- just map
  them to actions, 2019.

\bibitem{srivastava2021training}
Rupesh~Kumar Srivastava, Pranav Shyam, Filipe Mutz, Wojciech Jaśkowski, and
  Jürgen Schmidhuber.
\newblock Training agents using upside-down reinforcement learning, 2019.

\bibitem{strupl2021reward}
Miroslav Strupl, Francesco Faccio, Dylan~R. Ashley, Rupesh~Kumar Srivastava,
  and J{\"u}rgen Schmidhuber.
\newblock Reward-weighted regression converges to a global optimum.
\newblock {\em ArXiv}, abs/2107.09088, 2021.

\end{thebibliography}
